\documentclass[10pt, conference]{ieeeconf}      

\IEEEoverridecommandlockouts                              

\newtheorem{theorem}{Theorem}[section]
\newtheorem{lemma}[theorem]{Lemma}
\newtheorem{proposition}[theorem]{Proposition}
\newtheorem{assumption}[theorem]{Assumption}

\usepackage{url}

\newcommand{\qed}{\nobreak \ifvmode \relax \else
      \ifdim\lastskip<1.5em \hskip-\lastskip
      \hskip1.5em plus0em minus0.5em \fi \nobreak
      \vrule height0.75em width0.5em depth0.25em\fi}

\usepackage{amsmath,amssymb}

\usepackage{algorithm,algorithmic}
\usepackage{tabularx}
\usepackage{tikz,hyperref,graphicx,units}
\usepackage{subfigure}
\usepackage{benktools}
\usepackage{bbm}
\renewcommand{\baselinestretch}{.5}

\usepackage{caption}
\usepackage{epstopdf}

\usepackage{sidecap,wrapfig}
\DeclareMathOperator*{\argmin}{arg\,min}

\newcolumntype{L}[1]{>{\RaggedRight\hspace{0pt}}p{#1}}
\newcolumntype{R}[1]{>{\RaggedLeft\hspace{0pt}}p{#1}}

\newboolean{include-notes}
\setboolean{include-notes}{true}
\newcommand{\adnote}[1]{\ifthenelse{ \boolean{include-notes}}%
 {\textcolor{blue}{\textbf{#1}}}{}}
 
 \newcommand{\sknote}[1]{\ifthenelse{ \boolean{include-notes}}%
 {\textcolor{blue}{\textbf{SK: #1}}}{}}
 
  \newcommand{\mlnote}[1]{\ifthenelse{ \boolean{include-notes}}%
 {\textcolor{purple}{\textbf{ML: #1}}}{}}
 
 \newcommand{\jmnote}[1]{\ifthenelse{ \boolean{include-notes}}%
 {\textcolor{orange}{\textbf{JM: #1}}}{}}

\renewcommand{\baselinestretch}{0.94}
\usepackage{times}
\usepackage{microtype}

\title{Constraint Estimation and Derivative-Free Recovery \\ for Robot Learning from Demonstrations}

\author{Jonathan Lee$^{1}$, Michael Laskey$^{1}$, Roy Fox$^{1}$, Ken Goldberg$^{1,2}$
\thanks{$^{1}$Department of Electrical Engineering and Computer Science}%
\thanks{$^{2}$Department of Industrial Engineering and Operations Research}
\thanks{$^{1-2}$The AUTOLAB at UC Berkeley; Berkeley, CA 94720, USA}
\thanks{{\tt\small jonathan\_lee@berkeley.edu, laskeymd@berkeley.edu, royf@berkeley.edu, goldberg@berkeley.edu}}%
}

\begin{document}

\maketitle
\global\csname @topnum\endcsname 0
\global\csname @botnum\endcsname 0

\thispagestyle{empty}
\pagestyle{empty}



\begin{abstract}

Learning from human demonstrations can facilitate automation but is risky because the execution of the learned policy might lead to collisions and other failures. Adding explicit constraints to avoid unsafe states is generally not possible when the state representations are complex. Furthermore, enforcing these constraints during execution of the learned policy can be challenging in environments where dynamics are difficult to model such as push mechanics in grasping. 
In this paper, we propose Derivative-Free Recovery (DFR), a two-phase method for generating robust policies from demonstrations in robotic manipulation tasks where the system comes to rest at each time step.
In the first phase, we use support estimation of supervisor demonstrations and treat the support as implicit constraints on states. We also propose a time-varying modification for sequential tasks. 
In the second phase, we use this support estimate to derive a switching policy that employs the learned policy in the interior of the support and switches to a recovery policy to steer the robot away from the boundary of the support if it drifts too close.
We present additional conditions, which linearly bound the difference in state at each time step by the magnitude of control, allowing us to prove that the robot will not violate the constraints using the recovery policy. A simulated pushing task in MuJoCo suggests that DFR can reduce collisions by 83\%. On a physical line tracking task using a da Vinci Surgical Robot and a moving Stewart platform, DFR reduced collisions by 84\%.
\end{abstract}

\section{Introduction} 

\begin{figure}
\center
\includegraphics[width=0.42\textwidth]{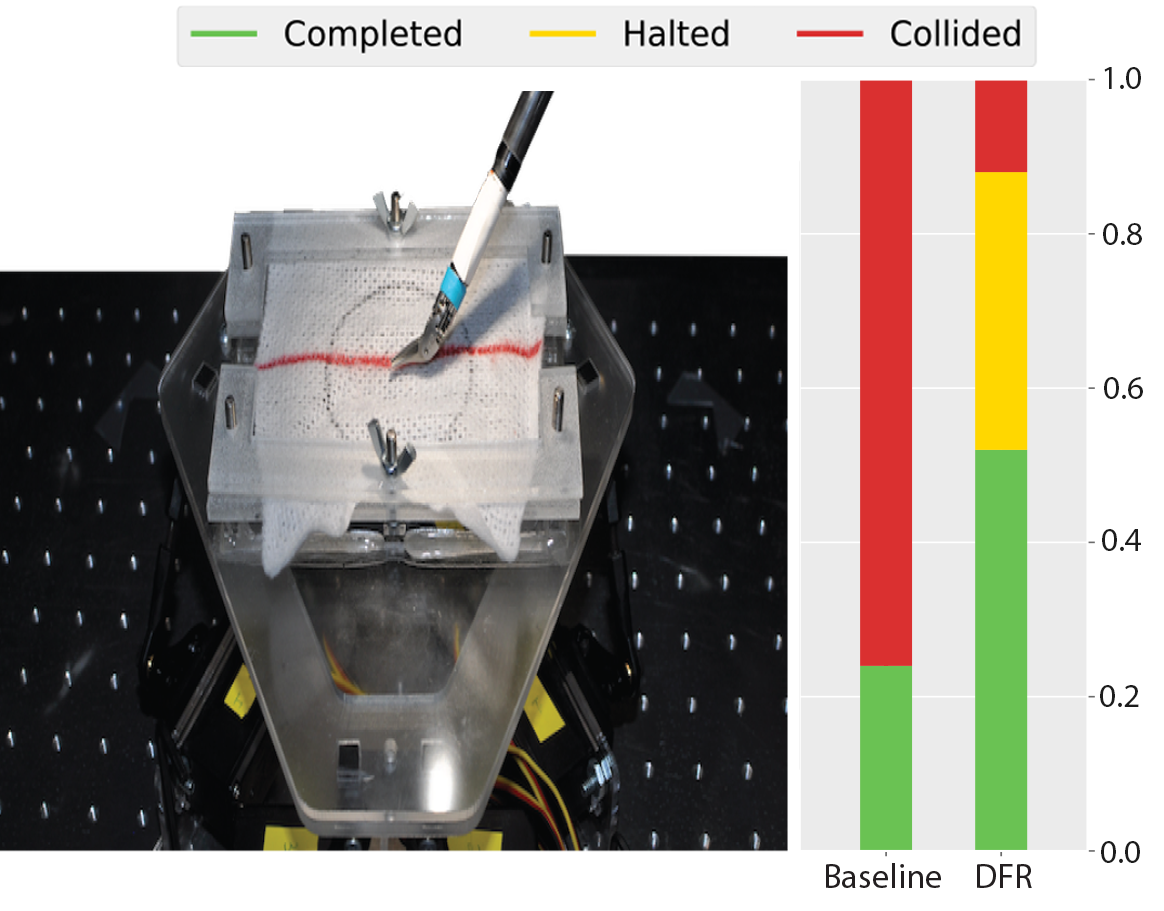}
\caption{
    \footnotesize
 The da Vinci Surgical Robot tracking a line drawn on gauze as the Stewart platform applies physical disturbances. The Baseline policy is compared with the policy with Derivative-Free Recovery (DFR) on the da Vinci line tracking task. Each segment depicts the fraction of ``Completed," ``Halted," and ``Collided" trajectories. The results show that DFR significantly reduces collisions while also increasing the fraction of completed trajectories.} \vspace*{-20pt}
\label{teaser}
\end{figure}

Robotic manipulation tasks are relevant in many industrial applications such as warehouse order fulfillment and flexible manufacturing where a robot must grasp or manipulate an object in environments with little structure. One method of approaching these problems is to construct an analytic model; however, doing so can often be difficult due to complex state spaces such as images, complicated mechanics such as pushing, and uncertainties in parameters such as friction. An alternative method is to use supervisor demonstrations to learn a policy. With learning from demonstrations, a robot observes a supervisor policy and learns a mapping from state to control via regression. This approach has shown promise for automation and robotics tasks such as grasping in clutter \cite{laskeyrobot}, robot-assisted surgery \cite{van2010superhuman}, and quadrotor flight \cite{coates2008learning}.
	
Enforcing constraints on states, such as ensuring that a robot does not tension  tissue above a certain level of force during a surgical task, remains an open problem in learning from demonstrations. Even if the demonstrated trajectories satisfy the constraints, there is no guarantee that the resulting learned policy will. For example, the robot may take a series of slightly sub-optimal actions due to approximation error of the learned policy and find itself in states vastly different from those visited by the supervisor. We desire to ensure the robot does not enter constraint-violating regions during execution. In this paper, we consider this problem for robotic manipulation in domains where the system comes to rest at each time step. This problem setting is inherent in many manipulation tasks in industrial and surgical settings with position control and has become increasingly important in automation \cite{chen2017algorithm, rossano2013easy, lu2015human}.

While techniques exist to enforce constraints on learned policies, they are often limited to operate in domains with known models \cite{gillula2012guaranteed, perez2017c}. This can be challenging when dealing with robotic manipulation where interactions between objects can be fundamentally hard to model \cite{thananjeyan2017multilateral}. It can also be challenging to explicitly specify constraints. In a surgical task, objects such as tissue are often soft and deformable and observations often come from images from an endoscope. Additionally, specifying constraints such as the level of tension allowed on certain piece of tissue may require hard-coding rules that rely on complex models of these objects and noisy observations. However, the supervisor's demonstrated data provide not only information about the desired policy, but also information about the constraints. Intuitively, the robot should only visit states that the supervisor knows are safe to visit. 

We propose leveraging the demonstration data to estimate the support of the supervisor's state distribution and treating the estimated support as a set of implicit constraints. The support is defined as the subset of the state space that the supervisor has non-zero probability of visiting. This subset is informative because it describes regions that must be safe since the supervisor visits those states. The complement of the support describes the region that may not be safe or include constraint-violating states. In the aforementioned surgical task, this would correspond to the robot recognizing that observations of heavily tensioned tissue are uncommon or nonexistant in the supervisor demonstrations and so it should try to avoid these states.

Various methods exist for density estimation which may be used to identify regions of support. In prior work, it was shown that the One Class SVM can be used effectively to estimate boundaries around the supervisor's demonstrations \cite{laskey2016shiv}.

We use this support estimate to derive a switching policy that employs the robot's learned policy in safe states and switches to a recovery policy if the robot drifts close to the boundary of the estimated support. The recovery policy is posed as a derivative-free optimization (DFO) of the decision function of the support estimator, which provides a signal towards estimated safe areas. Because traditional DFO methods can be difficult to apply in dynamical systems, we propose a method to find likely directions toward safety by examining the outcome of applying small perturbations in the control signal, which we assumed lead to small changes in state. The recovery policy is designed to steer the robot towards safer regions in the best case or come to a stop if it cannot. We also present a condition, which bounds the change in state with respect to the magnitude of control, under which the robot will never enter the constraint-violating regions using the recovery policy.

In simulated experiments on the MuJoCo \textit{Pusher} task \cite{hausman2017multi, singh2017gplac}, we compared the proposed recovery control to a naive baseline and found that recovery reduced performance of the learned policy by 35\% but also reduced the rate of collisions by 83\%.

We also deployed the recovery strategy on a da Vinci Surgical Robot in a line tracking task under disturbances from a Stewart platform shown in Fig. \ref{teaser}(b) and found that the successes increased from 24\% to 52\%  and collisions decreased from 76\% to 12\%.

This paper makes four contributions:
\begin{enumerate}
    \item An implicit constraint inference method using support estimation on demonstrated data.
    \item Derivative-Free Recovery, a novel model-free method for recovery control during execution of a learned policy.
    \item Conditions under which the robot will not violate the constraints while using the recovery method.
    \item Experimental results evaluating the proposed methods in simulation and on a physical robot.

\end{enumerate}

\section{Related Work}

\textbf{Learning from Demonstrations in Automation Tasks:} Learning from demonstrations, sometimes also referred to as imitation learning, describes a broad collection of methods for learning to replicate sequential decision making. Specifically in automation and robotics, learning from demonstrations often makes use of kinesthetic or teleoperated demonstrations of control given by a human supervisor that is able to reason about the task from a high level. The learning system takes as input these demonstrations and outputs a policy mapping states to actions.

Prior work in automation has explored learning from demonstrations for highly unstructured tasks such as grasping in clutter, scooping, and pipetting \cite{laskeyrobot, liang2017dvrk}. Past work has also addressed the specific problem of learning from demonstrations under constraints \cite{billard2008robot, calinon2009robot}. A popular method for dealing with unknown constraints is to identify essential components of multiple successful trajectories based on variances in the corresponding states and then to produce a learned policy that also exhibits those components \cite{calinon2008probabilistic}. Despite early empirical success, constraint satisfaction is not guaranteed \cite{perez2017c} and the machine learning model used to learn the policy must often be compatible with the variance estimator. We consider a method that is agnostic to the machine learning model.

C-LEARN \cite{perez2017c} successfully incorporated motion planning with geometric constraints into keyframe-based learning from demonstrations for manipulation tasks, guaranteeing constraint satisfaction. However, constraints must be inferred from predetermined criteria, and an accurate model is required in order to satisfy those constraints using a motion planner.

Recent work has also dealt with learning constraint satisfaction policies from demonstrations when the constraints are unknown but linear with respect to the controls \cite{armesto2017learning, 
howard2009novel}. There has also been recent work in guiding model-free policies towards states about which they are more confident, effectively trying to avoid certain unknown regions of the state space via temporal difference learning \cite{schroecker2017state}.

Significant literature exists on the topic of error detection and recovery (EDR) \cite{donald1989error} with models. For example, Donald et al. \cite{donald2013planning} used EDR methods for planning with microrobots.  In this paper we address this problem in the model-free domain.

\textbf{Safe Learning to Control:} Interest in learning-based approaches for control for under constraints has increased as a result of recent advances in learning and policy search, which have traditionally been studied without constraints due to their exploratory and unpredictable nature \cite{achiam2017constrained}.

Assuming dynamics are known or can be estimated, Gillula and Tomlin \cite{gillula2012guaranteed} applied reachability analysis to address bounded disturbances by computing a sub-region within a predefined safe region where the robot will remain safe under any disturbance for a finite horizon. This region is referred to as the ``discriminating kernel" by Akametalu et al. \cite{akametalu2014reachability} and Fisac et al. \cite{fisac2017safety} who extended this theory to obtain safe policies that are less conservative under uncertainty. In their work, the safety controller is applied only on the boundary of the discriminating kernel while the robot's controller is freely applied in the interior, resulting in a switching policy. Although our objectives are similar, there are several key differences in our assumptions. First, we do not require the model or constraints to be specified explicitly to the robot. Also, safe reinforcement learning aims to facilitate exploration for policy improvement while our approach addresses safe execution of policies after learning.

In surgical robotics, Yip and Camarillo \cite{yip2016model} studied model-free control of continuum manipulators in constrained environments where the constraints are initially unknown. The authors proposed a combined position and force controller which actively estimates Jacobians. Continuum manipulators in surgical environments are in general designed to ``conform" to obstacles constraints. In this paper, we consider manipulators in general constrained environments where the manipulator may not have direct force feedback from interacting with constraints.

\section{Problem Statement}

\textbf{Assumptions: } We consider a discrete-time manipulation task with an unknown Markovian transition distribution and constraints specifying stay-out regions of the state space, such as collisions. The constraints are initially unknown to the robot. We further assume that the system comes to rest at each time step as in manipulation tasks with position control such as \cite{liang2017dvrk}. As in many applications of learning from demonstrations, we do not assume access to a reward function, meaning that there is no signal from the environment to indicate whether the robot is successfully completing the task. We assume a given set of observations of demonstrations from a supervisor that do not violate the constraints. The remainder of this section formalizes and elaborates these assumptions.

\textbf{Modelling: }  Let the continuous state space and continuous control space be denoted by $\mathcal X \subseteq \mathbb R^n$ and $\mathcal U \subseteq \mathbb R^d$, respectively. The unknown transition distribution is given by $p(x_{t + 1} | x_t, u_t)$ with unknown initial state distribution $p_0(x)$. We define $\tau = \left\{ (x_0, u_0), \ldots, (x_{T - 1}, u_{T - 1}), (x_T)\right\}$ as a trajectory of state-action pairs over $T$ time steps. The probability of a trajectory under a stochastic policy $\pi : \mathcal X \mapsto \mathcal U$ is given by
\begin{align*}
p(\tau | \pi) = p_0(x) \prod^{T - 1}_{t = 0} p(u_t | x_t ; \pi) p(x_{t + 1} | x_t, u_t).
\end{align*}
Additionally, we denote $p_t(x ; \pi)$ as the distribution of states at time $t$ under $\pi$, and we let $p(x ; \pi) = \frac{1}{T}\sum_{t = 0}^T p_t(x; \pi)$. 

Although unknown, the dynamics of the system are assumed to leave the system at rest in each time step. For many practical discrete-time manipulation tasks, this property is common for example  in settings where controls are positional and objects are naturally at rest such as in grasping in clutter \cite{laskeyrobot}. 

\textbf{Objective: } This paper considers the problem of learning to accomplish a manipulation task reliably from observed supervisor demonstrations while attempting to satisfy constraints. We will only consider learning from demonstrations via direct policy learning, i.e. supervised learning.

Instead of a reward function, we assume that we have a supervisor that is able to demonstrate examples of the desired behavior in the form of trajectories. The robot's goal is then to replicate the behavior of the supervisor.

The goal in direct policy learning is to learn a policy $\pi: \mathcal X \mapsto \mathcal U$ that minimizes the following objective
\begin{equation}\label{obj}
\mathbb E_{\tau \sim p(\tau | \pi) } \: J(\tau, \pi^*)
\end{equation}
where $J(\tau, \pi^*)$ is the cumulative loss of trajectory $\tau$ with respect to the supervisor policy $\pi^*$:
\begin{align}
    J(\tau, \pi^*) := \sum_{t = 0}^{T - 1} \ell (u_t, \pi^*({x_t})).
\end{align}
$\pi^*(x_t)$ indicates the supervisor's desired control at the state at time $t$, and $\ell: \mathcal U \times \mathcal U \mapsto \left[ 0, \infty \right)$ is a user-defined, non-negative loss function, such as the Euclidean norm of the difference between the controls. Note that in (\ref{obj}), the expectation is taken over trajectories sampled from $\pi$. Ideally, the learned policy minimizes the expected loss between its own controls and those of the supervisor on trajectories sampled from itself. 

This objective is difficult to optimize directly because the trajectory distribution and loss terms are coupled. Instead, as in \cite{laskey2016shiv, ross2010efficient}, we formulate it as a supervised learning problem:
\begin{equation}\label{surr}
\min_{\pi \in \Pi} \quad \mathbb E_{\tau \sim p(\tau | \pi^*)} J(\tau, \pi).
\end{equation}
Here, the expectation is taken with respect to the trajectories under the supervisor policy, rather than the robot's policy. This formulation decouples the distribution and the loss, allowing us to collect a dataset of training demonstrations $\left\{ \tau_1, \ldots, \tau_N \right\}$ from the supervisor and minimize the empirical loss to obtain a learned policy $\hat \pi$:
\begin{equation}\label{emp}
\hat \pi = \underset{\pi \in \Pi}{\argmin} \quad \frac{1}{N} \sum_{i  =1} ^N J(\tau_i, \pi).
\end{equation} 

This relaxation of the problem comes with a consequence. Because the training dataset is sampled from a different distribution (the supervisor distribution), it is difficult to apply traditional supervised learning guarantees about the learned policy. This problem is referred to as \textit{covariate shift}. Prior work has considered learning recovery behavior during training \cite{ross2010efficient, laskey2017dart}, but it is still not clear how errors may affect the robot or its environment, which motivates the need for increased robustness during execution.

\textbf{Constraints: } While prior work in learning from demonstrations has often dealt in the unconstrained setting, we consider learning in the presence of constraints that specify regions of the state space that the robot should actively avoid. Using the notation of \cite{akametalu2014reachability}, let $\mathcal K$ be a subset of $\mathcal X$ that is constraint-satisfying and let $\mathcal K^C$, the constraint-violating region, be its relative complement in $\mathcal X$. Note that this region is different from the support of the supervisor. The support is a subset of $\mathcal K$ that does not intersect $\mathcal K^C$. The supervisor, who is able to reason about the task at a high level, demonstrates the task robustly by providing constraint-satisfying trajectories during training time only. That is, $p(x; \pi^*) = 0$ for all $x \in \mathcal K^C$. Our objective is to have the robot learn this policy from demonstrations and perform it autonomously and reliably without entering the constraint-violating regions when it is deployed.

\section{Algorithms}

\subsection{Support Estimation}
Given a set of sample states from supervisor demonstrations, $\left\{ x_i \right\}_{i = 1}^n \subset \mathcal X$, support estimation returns an approximate region of non-zero probability, $\left\{x \in \mathcal X \ : \ p(x; \pi^*) > 0 \right\}$. Since the supervisor is always safely demonstrating the task, if $p(x; \pi^*) >0$, then we know that $x \in \mathcal K$.

As presented by Sch\"olkopf et al. in \cite{scholkopf2001estimating}, a common objective in support estimation is to identify the set in the state space of least volume that captures a certain probability threshold $\alpha$. For Lebesgue measure $\mu$ and probability space $(\mathcal X, \mathcal B, P)$ where $\mathcal B$ is the set of measurable subsets of $\mathcal X$ and $P_{\pi^*}(B)$ is the probability of $B \in \mathcal B$ under the supervisor policy, the \textit{quantile function} is
\begin{align*}
U(\alpha) = \inf_{B \in \mathcal B} \left\{ \mu(B) \: : \: P_{\pi^*}(B) \geq \alpha \right\}.
\end{align*}
The minimum volume estimator, $B(\alpha)$, is defined as the subset that achieves this objective for a given $\alpha$ \cite{scholkopf2001estimating}.
To obtain the true support, we set $\alpha = 1$ since we would like to obtain the minimum volume estimator of the entire non-zero density region. In practice, there is no way to obtain the true minimum volume estimator with finite data and an unknown distribution. Instead, many methods for obtaining approximate support estimates have been proposed \cite{gayraud1997estimation, scholkopf2001estimating}. For example, one might employ a kernel density estimator. In these cases, we often let $\alpha < 1$ to allow some tolerance for outliers, so that the estimator is more robust.

Despite prior use of support estimation in robotic and sequential tasks \cite{laskey2016shiv}, estimators for which $\alpha < 1$ can be problematic when applied directly to observed states due to the time-variant nature of the state distribution. We provide a simple example where the minimum volume estimator fails to provide an accurate support estimate.

Consider two disjoint subsets of the state space $B_0$ and $B_1$, such that $p_0(x \in B_0; \pi^*) = 1$ and $p_t(x \in B_1; \pi^*) = 1$ for all $t > 0$. It is clear that $\lim_{T \rightarrow \infty} p(x \in B_0; \pi^*) = \lim_{t \rightarrow \infty} \frac{1}{T} \sum_{t = 0}^T p_t(x \in B_0; \pi^*) = 0$ since states in $B_0$ are only possible as initial states. Therefore, if we simply draw examples from the distribution $p(x;\pi^*)$, the appropriate minimum volume estimate of any $\alpha$-quantile  will not include $B_0$ because the entire long-term probability density lies entirely in $B_1$.

This example reveals an important problem in the support estimation for tasks involving Markov chains: regions of the state space may be left out of the support estimate not because they are not relevant,  but rather they are only relevant in a vanishing fraction of time steps. Thus, even if a region is known to surely be in the supervisor trajectories at some time step, it may be excluded from the estimated support. The example is not unrealistic. This problem may occur, albeit less severely, in any Markov chain where regions of the state space are revisited at different time steps.

Taking inspiration from \cite{ross2010efficient}, instead of using a single support estimator to encompass the entire distribution over states $p(x; \pi^*)$, we propose to use $T$ estimators each for a corresponding distribution $p_t(x; \pi^*)$. By doing so, we limit each estimator to a single time step potentially reducing sample variance. When demonstrations are time-aligned, this can lead to improved support estimation. When they are not, we at worst increase the sample complexity $T$-fold.

In this paper, we use the One Class Support Vector Machine (OCSVM) to estimate the support \cite{scholkopf2001estimating, scholkopf2002learning}. The estimator determines a small region of $\mathcal X$ where the fraction of examples within the region converges to an appropriate $\alpha$-quantile as more data is collected \cite{vert2006consistency}. 
Sch\"olkopf et al. \cite{scholkopf2001estimating} present the primal optimization problem of the OCSVM as
\begin{align*}
\min_{w, \rho, \epsilon} \quad & \frac{1}{2} \|w\|_2^2 + \frac{1}{\nu m} \sum_{i = 1}^m \epsilon_i - \rho \\
\text{s.t.} \quad & w^\top \phi(x_i) \geq \rho - \epsilon_i \quad i = 1, \ldots, m
\end{align*}
where $m$ is the number of training examples, $0 < \nu < 1$ is a hyperparameter used to adjust the quantile level, and $\phi(\cdot)$ is a mapping from the state space to some implicit feature space. 

At run time, we can determine whether each visited state lies in the estimated support by evaluating $\text{sgn}\left\{g(x)\right\}$, where $g(x) = w^\top \phi(x) - \rho$ is the decision function. Positive values indicate that $x$ is in the estimated support and negative values indicate otherwise. For the remainder of this paper, we will use the Gaussian kernel: $\phi(x)^\top \phi(x') = e^{-\gamma \|x - x'\|_2^2}$.

\subsection{Derivative-Free Recovery Control}

\begin{figure*}
\center
\includegraphics[width=.9\textwidth]{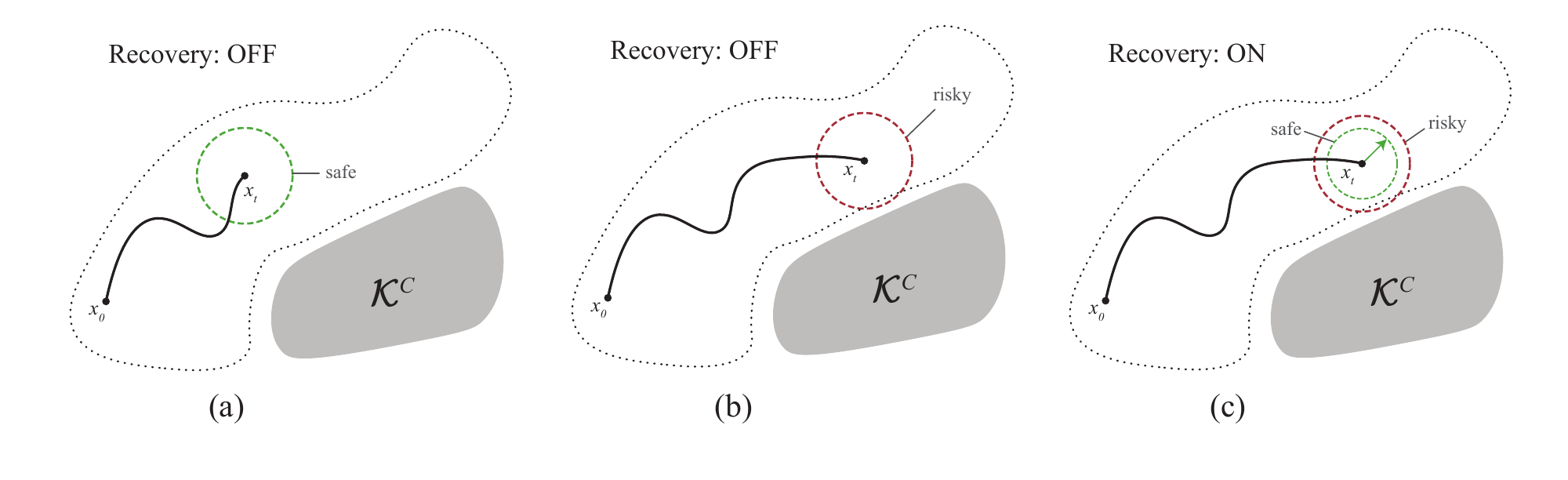}
\vspace*{-20pt}
\caption{
    \footnotesize
The estimated support is represented as the dotted shape and the region of constraint-violating states is denoted by $\mathcal K^C$. At run time, the robot executes its learned policy starting at state $x_0$. The dashed circle around the current state $x_t$ indicates the ball of states that the robot may enter in the next time step given its intended action. In (a), the ball is fully contained in the estimated support, so the robot uses its learned policy only. In (b), the ball overlaps with the boundary of the estimated support, indicating that the next state may be unsafe. In (c), as a result the recovery policy is activated, restricting the magnitude of control, as random perturbations are applied to find a direction of ascent.}
\label{alg}
\vspace*{-15pt}

\end{figure*}

Once the support has been identified based on the supervisor demonstrations, the robot must  learn a policy that minimizes the loss while staying within the boundaries of the estimated support to ensure it does not violate the constraints. To reconcile these potentially competing objectives, we propose using a switching policy at run time as in \cite{akametalu2014reachability} that alternates between the learned policy $\hat \pi$ from (\ref{emp}) and a recovery policy $\pi_{R}$ that attempts to guide the robot to interior regions of the support if it is close to the boundary.

The decision functions of the support estimators provide natural signed distance functions to the boundary of the estimated support. Thus as the robot rolls out, we can obtain reasonable online estimates of how ``close" it is to the boundary. If the robot is in a state with a relatively high decision function value, it should apply its learned controls freely. However, if the decision function value at the robot's state is close to zero (i.e. near the boundary), the recovery should be activated to help the robot recover.

Formally, we may define a ``close" distance as any distance from the boundary where the robot's learned policy could send it past the boundary in the next time step. Without a  model of the dynamics, this cannot be known exactly. We introduce a tuneable hyperparameter $\lambda$, similar to a learning rate, which intuitively corresponds to a proportional relationship between the amount of change in the decision function and the magnitude of the applied control. We then propose a switching policy $\tilde \pi$ to incorporate the recovery behavior $\pi_{R}$:
\begin{align*}
\tilde \pi = \begin{cases}
	\hat \pi & g_t(x_t) > \lambda \|\hat \pi(x_t)\|_2 \\
	\pi_{R} & \text{otherwise}.
\end{cases}
\end{align*}

The simplest recovery behavior is to apply zero control for the remaining time steps after the threshold has been crossed, potentially leaving the task incomplete. While this strategy will in principle reduce the risk of entering a constraint-violating state, it is overly conservative.

To increase the chance of completing the task while maintaining constraint satisfaction, we propose a best-effort recovery policy that leverages the decision function of the support estimator. When enabled, the recovery policy should drive the robot towards regions of the state space where the estimated decision value is higher, indicating the interior regions of the support. That is, we want to ascend on $g_{t}(x)$. If the dynamics model were known analytically, we could apply standard optimization techniques such as gradient ascent to obtain a local maximum of the decision function with respect to the controls. However, the model-free domain considered in this paper presents a challenge, as the decision function with respect to the control is unknown. It is therefore not possible to use analytic derivative approaches to optimize the objective.

Additionally, conventional Derivative-Free Optimization (DFO) and finite difference methods \cite{rios2013derivative}, where multiple function evaluations of $g_t(x)$ would be made to find directions of ascent, are not suitable because we cannot directly manipulate the state $x$. Instead we may only control the state by applying input controls through the system, and we may only evaluate the effect of a control once it has been applied. Furthermore, because the system advances each time we apply a control, the objective function, which is a function of the current state, must change as well.

To address this problem, we propose a novel greedy derivative-free optimization approach, called Derivative-Free Recovery (DFR) Control, that employs a method similar to hill-climbing to make a best-effort recovery by applying conservative controls to ascend on the decision function. Consider the robot at state $x_t$. A small control perturbation $u_\delta$ is applied and yields a small change in state from $x_t$ to $x_\delta$. Consequently the perturbation also results in a small change in the decision function which indicates whether $u_\delta$ causes ascent or descent of the decision function at state $x_t$.

\newcommand{\uDFO}{u_\delta}
\newcommand{\urec}{u_{R}}

\begin{figure}
 \begin{algorithm}[H]
 \caption{Derivative-Free Recovery (DFR)}
 \begin{algorithmic}[1]\label{alg1}
  \STATE Initialize $t \gets 0$, $x_0 \sim p_0(x)$
  \WHILE {$t < T$}
  \STATE $\hat u_t \gets \hat \pi(x_t)$
    \WHILE{$g_t(x_t) \leq \lambda \|\hat u_t\|_2$}
    \STATE Sample random $\uDFO$ s.t. $\| \uDFO \|_2 \ll \frac{g_t(x_t)}{\lambda}$
    \STATE Apply $\uDFO$ and observe $x_\delta \sim p(\cdot | x_t, \uDFO)$
    \IF{$g_{t}(x_\delta) \leq g_t(x_t)$}
    \STATE $\uDFO \gets -\uDFO$
    \ENDIF
    \STATE $\urec \gets \eta \frac{\uDFO}{\|\uDFO\|_2} $
    \STATE Apply $\urec$ and observe $x \sim p\left(\cdot |  x_{\delta}, \urec\right)$
    \STATE $x_t \gets x$
    \STATE $\hat u_t \gets \hat \pi(x_t)$
    \ENDWHILE
    \STATE Apply $\hat u_t$ and observe $x_{t + 1} \sim p(\cdot | x_t, \hat u_t)$
    \STATE $t \gets t+1$
  \ENDWHILE
 \end{algorithmic}
 \end{algorithm}
 \vspace*{-30pt}
\end{figure}

The full procedure for applying recovery controls online is shown in Algorithm \ref{alg1}. At any given time step, a control is obtained from the robot's policy. Using $\lambda$ and the magnitude of the control, it is decided whether the robot's control is safe to use. If it is safe, then the control is executed without interruption. In the event that it is not safe, the recovery strategy is activated. A random but small control $u_\delta$ is then sampled, such that applying that control would still result in a positive decision function value. On lines 7 and 8, an approximate ascent direction is identified by executing the small random control and evaluating the decision function again. The recovery control $u_{R}$ is then chosen as a vector in the direction of ascent with conservative magnitude $\eta$, where $0 < \eta < \frac{g_t(x)}{\lambda}$, limiting the risk of steering the robot out of the support and potentially into constraint-violating regions. Thus a larger choice of $\lambda$ corresponds to a more conservative policy. While guaranteeing improvement of decision function may not be possible in all problems, improvements may be found in environments with locally nice and differentiable dynamics. A visual procedure is given in Fig. \ref{alg}.

Furthermore, a fail-safe strategy naturally follows from this algorithm. In the event that recovery is not possible and the robot gets arbitrarily close to the boundary of the support, the magnitudes of the sample and recovery controls approach zero, effectively halting the robot to prevent it from failing. In the next section, we present conditions when we can guarantee constraint satisfaction for Algorithm~\ref{alg1} and formalize a worst-case choice for $\lambda$.

\subsection{Conditions for Constraint Satisfaction}
While it is not strictly necessary for good performance on many manipulation tasks as seen in the experiments, we introduce a condition on the dynamics model specific to some systems that formally characterizes a notion that the system comes to rest between time steps and allows us to guarantee that the robot will not violate constraints in systems where it is satisfied.

\begin{assumption}
For all $t \in \left\{ 0, \ldots, T - 1 \right\}$ there exists some constant $K$ such that the following holds:
\begin{equation}\label{dyn}
    \|x_{t+1} - x_t\|_2 \leq K \|u_t\|_2.
\end{equation}
\end{assumption}
This condition holds in stable manipulation systems where the amount of change from one state to the next is limited.

We now show that under the proposed algorithm and the above condition, it is guaranteed that the robot will not violate the constraints. Formally, let $\tilde B_{t} \equiv \left\{x\::\:g_t(x) \geq 0\right\}$ be the estimated support of $p_t(x | \pi^*)$ with a corresponding $L$-Lipschitz decision function $g_t(x)$. By (\ref{dyn}) and the Lipschitz continuity of $g_t(x)$, $|g_t(x_{t + 1}) - g_t(x_{t})| \leq L \| x_{t + 1} - x_t\|_2 \leq LK \|u_t\|_2$. This inequality formalizes a worst-case change in decision function value with respect to the magnitude of the robot's control, giving concrete meaning to the choice of $\lambda = LK$. Next, we guarantee constraint satisfaction for states in the estimated support:

\begin{lemma}\label{insupport}
If at time $t$, the robot is in state $x_t$ and $g_t(x_t) \geq 0$ and $\tilde B_{t} \cap \mathcal K^C = \emptyset$, then $x_t \in \mathcal K$.

\end{lemma}
\begin{proof}
This follows immediately from the condition that $\tilde B_{t} \equiv \left\{x\::\:g_t(x) \geq 0\right\}$, which implies that $x_t \in \tilde B_{t}$. Thus, $x_t$ must be in $\mathcal K$.
\end{proof}

Using this lemma, we are able to establish the following proposition:

\begin{proposition}\label{bounddec}
Under Algorithm~\ref{alg1} and the preceding conditions, the robot is never in violation of the constraints if $\tilde B_{t} \cap \mathcal K^C$ is empty.
\end{proposition}

\begin{proof}
The proof is by induction. Assume that the robot starts inside the estimated support. The induction assumption is that $g_t(x_t)\ge 0$, and we prove that this remains true after each step.

In the case where the learned policy $\hat\pi$ is constraint-satisfying, $\|\hat u_t\|_2<\tfrac 1{LK}g_t(x_t)$, we apply this control, and the next state satisfies
\begin{align*}
g_{t+1}(x_{t+1})&\ge g_t(x_t)-LK\|u_t\|_2>0.
\end{align*}

The remaining case is where we switch to the recovery strategy, and we apply both $\uDFO$ and $\urec$ with
\begin{align*}
\|\uDFO\|_2&=\tfrac\epsilon{LK}g_t(x_t)\\
\|\urec\|_2&=\eta\le\tfrac{1-\epsilon}{LK}g_t(x_t)
\end{align*}
for some $0 < \epsilon \ll 1$ splitting the difference between $\eta$ and $\tfrac{g_t(x_t)}{LK}$. Then the state $x$ after applying these controls satisfies
\begin{align*}
g_t(x)&\ge g_t(x_t)-LK(\|\uDFO\|_2+\|\urec\|_2)\ge0.
\end{align*}

We have shown that always $g_t(x_t)\ge0$. If $\tilde B_t \cap \mathcal K ^C = \emptyset$, then by Lemma~\ref{insupport} the robot is always constraint-satisfying.
\end{proof}

The intuition behind the proof of this proposition is that if we choose DFR controls with appropriately small magnitudes, applying those controls will never lead to a step that exceeds the boundary of the estimated support.

\section{Experiments}

\begin{figure*}
\center
\includegraphics[width=1.0\textwidth]{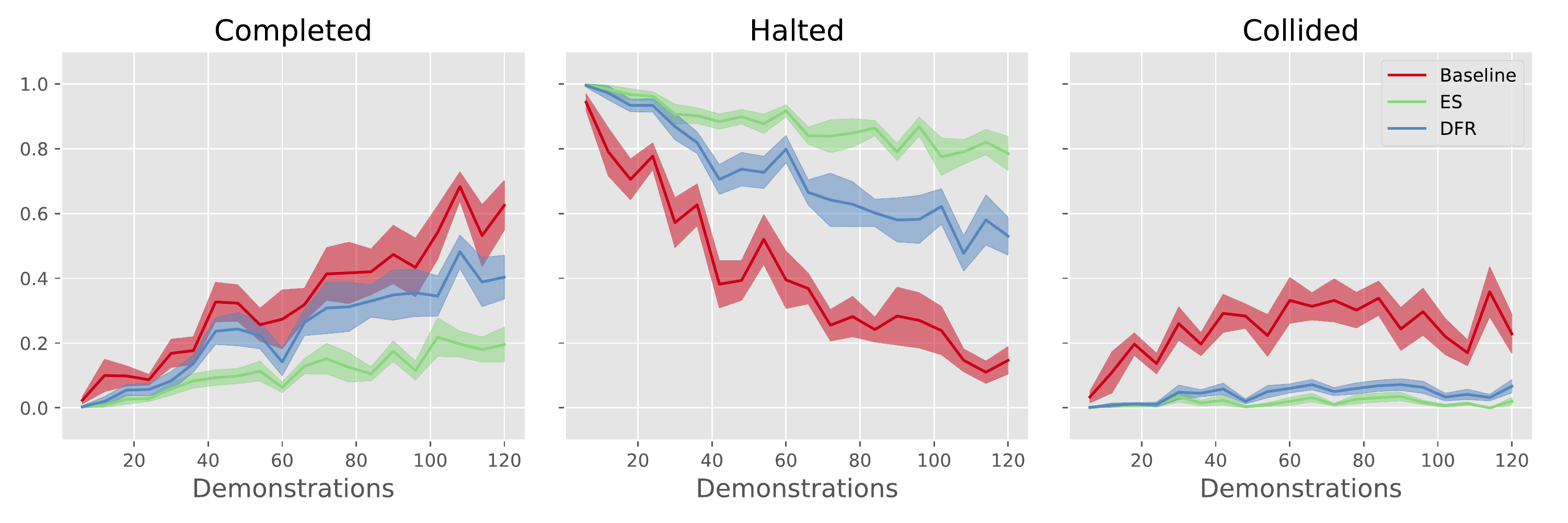}
\vspace*{-20pt}
\caption{
    \footnotesize
\textit{Left:} The fraction of completed completed samples of the three methods (Baseline, Early Stopping (ES), DFR) is plotted as a function of the number of demonstrations. DFR achieves a comparable completion rate to Baseline. \textit{Middle:} Halting rate which decreases for all methods as the learned policy acquires more data. Although Basline's halting rate decreases faster, it ultimately incurs more collisions without recovery. \textit{Right:} The collision rate for Baseline is much higher than either ES or DFR, which both have consistently low collision rates even with very little data.
}
\label{graph}
\vspace*{-15pt}

\end{figure*}

We conducted manipulation experiments in simulation and on a physical robot to evaluate the proposed detection method and the reliability of various recovery strategies. Our experiments aim to answer the following questions:
\begin{enumerate}
    \item Does support estimation provide a viable method for inferring safe regions given supervisor demonstrations when real constraint-violating regions exist but are not explicitly programmed by the supervisor? Is it viable even on systems where the conditions for constraint satisfaction do not necessarily hold?
    \item Does DFR effectively climb the decision function?
    \item How does DFR perform when varying the number of trajectories demonstrated?
    \item How does DFR perform in response to small disturbances not seen during training time?
\end{enumerate}

\subsection{Pusher Simulation}

\begin{figure}
\center
\includegraphics[width=0.45\textwidth]{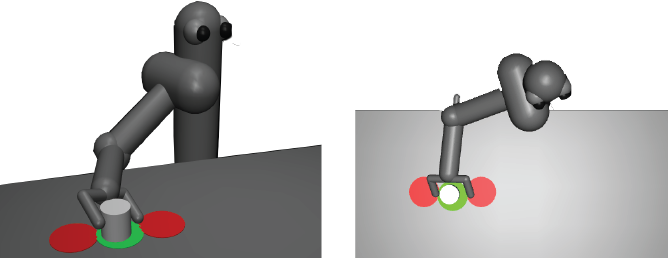}
\caption{
    \footnotesize \textit{Left:} The \textit{Pusher} task. The robot must learn to push the light gray object over the green circle without crossing over the red circles. \textit{Right:} Top-down view.
}
\vspace*{-20pt}
\label{fig:pusher}
\end{figure}

\textit{Pusher} (Fig. \ref{fig:pusher}) is an environment simulated in MuJoCo \cite{todorov2012mujoco} that considers the task of a one-armed robot pushing a light gray cylinder on a table to a green goal location. The initial state of the cylinder varies with each episode, preventing the robot from simply replaying a reference trajectory to succeed.

The robot has seven degrees of freedom controlling joint angle velocities. The state space consists of the joint angles, the joint angle velocities and the locations of the cylinder, end-effector, and goal object in 3D space. We modified the original task to allow control via direct changes in pose as opposed to velocity control of the joint angles. That is, the objects have no lasting momentum effects. We also introduced two regions marked in red representing the constraints of the task. The robot and the cylinder should not collide with these red regions. We stress that the robot does not know to avoid collisions with these states a priori, but the supervisor does. The robot must learn the support of the supervisor in order to recover if it approaches the collision states.

We generated an algorithmic supervisor using Trust Region Policy Optimization \cite{schulman2015trust} to collect large batches of supervisor demonstrations. The learning model used a neural network with two 64-node hidden layers and $\tanh$ activations. 120 supervisor trajectories were collected for each trial. The learning models were also represented with neural networks optimizing (\ref{emp}). The models cannot match the supervisor exactly, which introduces the need for the recovery policy.

For the OCSVM, we set $\nu = 0.05$ as an arbitrary quantile of the observed data and then tuned the kernel scale $\gamma = 5.0$ on out-of-sample trajectories from the supervisor. To simplify the support estimation, we removed joint angles from the state space to include only those features relevant to the recovery behavior, as we found extraneous features often caused the OCSVM to require much more data.

For this task, we define a ``Completed" trajectory to be any trajectory that reached the goal state without colliding. This includes trajectories where recovery was successful. A ``Collided" trajectory is any trajectory that reached a collision state. Finally, a trajectory that ``Halted" is any trajectory that neither reached the goal state nor entered a collision state in the allotted time. For example, the recovery policy may intentionally halt the task in high risk situations, resulting in a constraint-satisfying but incomplete trajectory. Trajectories that halted are strictly preferable to collisions. In many practical cases, they can also be reset, and the task may be attempted again. The ideal policy should minimize collisions while maintaining a high rate of completion.

We compared the proposed recovery strategy (DFR) in Algorithm \ref{alg1} to a Baseline, which did not employ any recovery behavior, and an early stopping (ES) policy, which simply halted when it came close to the estimated support boundary.  Fig. \ref{graph} illustrates the completed, halted, and collision rates for each method while varying the number of demonstrations of data. Across 10 trials with 60 evaluation samples per data-point per trial, DFR and ES significantly reduced the collision rate even with very little data compared to the Baseline, suggesting that staying within the estimated support is a viable method to avoid entering constraint violating regions. As more data was added, the completion rates of all three increased; however, DFR recovered from high risk situations allowing it to surpass ES and reach a comparable completion performance to the Baseline without significant collisions. DFR on average over all iterations achieved 83\% fewer collisions compared to the Baseline. Additionally the completion rate of DFR was only 65\% of that of the Baseline. Note that, due its conservative controller, DFR can prolong the wall clock time of a trajectory requiring an average of 1.50 seconds per trajectory while the baseline and ES required 0.13 seconds and 0.07 seconds, respectively. 

Fig. \ref{fig:dvrkconverge} depicts the effectiveness of the derivative-free optimization technique on the decision function when the recovery strategy is activated. Note that the recovery strategy remains activated until the value of the decision function reaches the cutoff value $\lambda \|\hat u_t\|_2$ or until 500 iterations have elapsed. On 50 instantiations of the optimization algorithm on \textit{Pusher}, each curve had nearly monotonic average improvement. We compared DFR with a finite difference oracle which was allowed to simulate controls before taking them in order to obtain numerical gradients with respect the controls.

\begin{figure}
\center
\includegraphics[width=0.4 \textwidth]{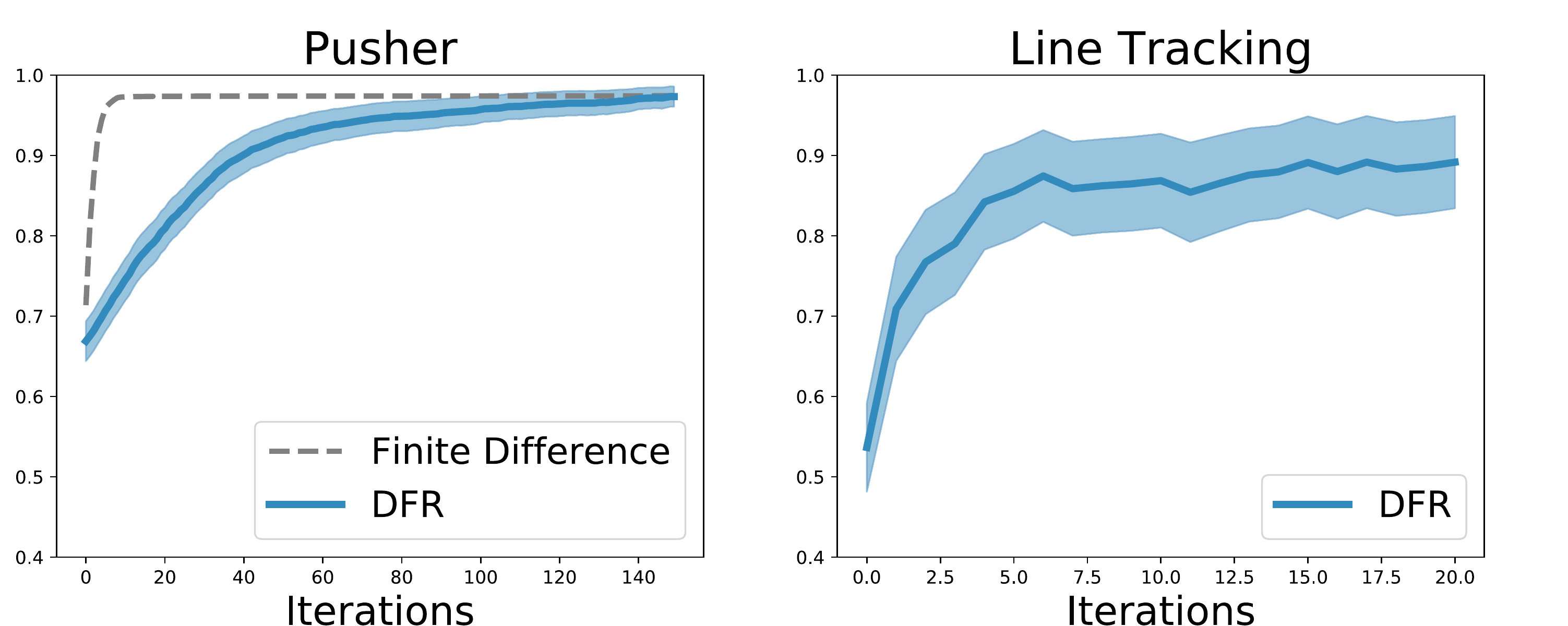}
\caption{
    \footnotesize
    \textit{Left:} The average of 50 DFR optimization curves on \textit{Pusher} is shown as a result of the recovery policy being activated during a trajectory. DFR is compared to a finite difference oracle. The decision function values were normalized between 0.0 and 1.0 where 1.0 represents the threshold of the switching policy. The normalized curves are capped at 1.0 because, by Alg. \ref{alg1}, the optimization stops once the threshold is reached. The few trajectories that do not reach 1.0 bring the average down slightly below 1.0 in both figures.
    \textit{Right:} The average of 30 DFR curves on the da Vinci.
    }
\vspace*{-20pt}
\label{fig:dvrkconverge}
\end{figure}

\subsection{Line Tracking on a da Vinci Surgical Robot}

Robotic surgical procedures consist of safety-critical tasks that require robust control due to disturbances in environment and dynamics that are difficult to model. We consider learning positional control in a task that mimics disturbances that might be encountered in such environments. We applied support estimation and recovery policies to the task of tracking lines on gauze using the Intuitive Surgical da Vinci robot as shown in Fig. \ref{teaser}. The objective of the task was to deploy a learned policy from demonstrations to follow a red line drawn in gauze using the end-effector under disturbances that were not shown during training time. The gauze was mounted on a Stewart platform \cite{patel2017sprk} which introduced random disturbances in the system during run time, but not during training. The robot used an overhead endoscope camera to observe images, which were processed to extract distances to the line and positions of the end-effector.

For this task, a ``Collided" trajectory was defined as any trajectory where the end-effector deviated by more than 4 mm from the red line. A ``Completed" trajectory was any trajectory that did not collide and tracked at least 40 mm of the gauze. All other trajectories were categorized as ``Halted."

Over 50 demonstrations were given with an open-loop controller without disturbances. Thus the trajectories never deviated from the line. As a result no notion of feedback control was present in the demonstration data. The robot's policy was represented by a neural network. As in \textit{Pusher}, we set the hyperparameters of the OCSVM by choosing a quantile level and validating on a held-out set of demonstrations.

The results are summarized in Fig. \ref{teaser}. The Baseline policy collided on the task repeatedly under random disturbances. The recovery was robust to the disturbances by attempting to keep the robot in the support. As in the \textit{Pusher} task, an increase in trajectories that halted was observed with DFR, indicating the ability to detect constraint-violating areas and halt in the worst case. An increase in the rate of completion was also observed as DFR applied controls to mitigate deviations from the line and resume the original policies when the state was sufficiently far from the boundary.


\section{Discussion and Future Work}
This paper presents Derivative-Free Recovery Control for robotic manipulation tasks. The results show that DFR can be used as an effective method of steering towards safe regions of a state space when a dynamics model is not known by ascending the decision function found by support estimation. Despite the promising asymptotic properties of the OCSVM, it can prove difficult in very high dimensional problems such as image space. This is a common trait of unsupervised learning the methods such as anomaly detection. Additionally the recovery procedure assumes the system comes to rest at each time step. In future work, we will extend DFR by addressing these problems with alternative support estimators and dimensionality reduction techniques and recovery planners that are less greedy.

\section{Acknowledgments}
This research was performed at the AUTOLAB at UC Berkeley in
affiliation with the Berkeley AI Research (BAIR) Lab, the Real-Time
Intelligent Secure Execution (RISE) Lab, and the CITRIS ``People and
Robots" (CPAR) Initiative and with UC Berkeley's Center for Automation and Learning for Medical Robotics (Cal-MR). The authors were supported in part by donations from Siemens, Google, Cisco, Autodesk, Amazon, Toyota Research, Samsung, and Knapp and by the Scalable Collaborative Human-Robot Learning (SCHooL) Project, NSF National Robotics Initiative Award 1734633, and by a major equipment grant from Intuitive Surgical. We thank our colleagues who provided thoughtful feedback and suggestions, in particular Bill DeRose, Sanjay Krishnan, Jeffrey Mahler, Matthew Matl, and Ajay Kumar Tanwani.

\bibliographystyle{IEEEtranS}
\bibliography{references}

\begin{thebibliography}{10}
\providecommand{\url}[1]{#1}
\csname url@samestyle\endcsname
\providecommand{\newblock}{\relax}
\providecommand{\bibinfo}[2]{#2}
\providecommand{\BIBentrySTDinterwordspacing}{\spaceskip=0pt\relax}
\providecommand{\BIBentryALTinterwordstretchfactor}{4}
\providecommand{\BIBentryALTinterwordspacing}{\spaceskip=\fontdimen2\font plus
\BIBentryALTinterwordstretchfactor\fontdimen3\font minus
  \fontdimen4\font\relax}
\providecommand{\BIBforeignlanguage}[2]{{%
\expandafter\ifx\csname l@#1\endcsname\relax
\typeout{** WARNING: IEEEtranS.bst: No hyphenation pattern has been}%
\typeout{** loaded for the language `#1'. Using the pattern for}%
\typeout{** the default language instead.}%
\else
\language=\csname l@#1\endcsname
\fi
#2}}
\providecommand{\BIBdecl}{\relax}
\BIBdecl

\bibitem{achiam2017constrained}
J.~Achiam, D.~Held, A.~Tamar, and P.~Abbeel, ``Constrained policy
  optimization,'' in \emph{International Conference on Machine Learning
  (ICML)}, 2017.

\bibitem{akametalu2014reachability}
A.~K. Akametalu, J.~F. Fisac, J.~H. Gillula, S.~Kaynama, M.~N. Zeilinger, and
  C.~J. Tomlin, ``Reachability-based safe learning with gaussian processes,''
  in \emph{IEEE Conference on Decision and Control (CDC)}, 2014.

\bibitem{armesto2017learning}
L.~Armesto, V.~Ivan, J.~Moura, A.~Sala, and S.~Vijayakumar, ``Learning
  constrained generalizable policies by demonstration,'' in \emph{Robotics:
  Science and Systems (RSS)}, 2017.

\bibitem{billard2008robot}
A.~Billard, S.~Calinon, R.~Dillmann, and S.~Schaal, ``Robot programming by
  demonstration,'' in \emph{Springer handbook of robotics}.\hskip 1em plus
  0.5em minus 0.4em\relax Springer Berlin Heidelberg, 2008, pp. 1371--1394.

\bibitem{calinon2009robot}
S.~Calinon, \emph{Robot programming by demonstration}.\hskip 1em plus 0.5em
  minus 0.4em\relax EPFL Press, 2009.

\bibitem{calinon2008probabilistic}
S.~Calinon and A.~Billard, ``A probabilistic programming by demonstration
  framework handling constraints in joint space and task space,'' in \emph{IEEE
  International Conference on Intelligent Robots and Systems (IROS)}, 2008.

\bibitem{chen2017algorithm}
C.~Chen, S.~Krishnan, M.~Laskey, R.~Fox, and K.~Goldberg, ``An algorithm and
  user study for teaching bilateral manipulation via iterated best response
  demonstrations,'' in \emph{International Conference on Automation Science and
  Engineering (CASE)}, 2017.

\bibitem{coates2008learning}
A.~Coates, P.~Abbeel, and A.~Y. Ng, ``Learning for control from multiple
  demonstrations,'' in \emph{International Conference on Machine Learning
  (ICML)}, 2008.

\bibitem{donald1989error}
B.~R. Donald, \emph{Error detection and recovery in robotics}.\hskip 1em plus
  0.5em minus 0.4em\relax Springer-Verlag New York, 1989.

\bibitem{donald2013planning}
B.~R. Donald, C.~G. Levey, I.~Paprotny, and D.~Rus, ``Planning and control for
  microassembly of structures composed of stress-engineered mems microrobots,''
  \emph{The International Journal of Robotics Research}, vol.~32, no.~2, pp.
  218--246, 2013.

\bibitem{fisac2017safety}
J.~F. Fisac, A.~K. Akametalu, M.~N. Zeilinger, S.~Kaynama, J.~H. Gillula, and
  C.~J. Tomlin, ``A general safety framework for learning-based control in
  uncertain robotic systems,'' \emph{arXiv preprint}, vol. abs/1705.01292,
  2017.

\bibitem{gayraud1997estimation}
G.~Gayraud, ``Estimation of functionals of density support,''
  \emph{Mathematical Methods of Statistics}, vol.~6, no.~1, pp. 26--46, 1997.

\bibitem{gillula2012guaranteed}
J.~H. Gillula and C.~J. Tomlin, ``Guaranteed safe online learning via
  reachability: tracking a ground target using a quadrotor,'' in \emph{IEEE
  International Conferece on Robotics and Automation (ICRA)}, 2012.

\bibitem{hausman2017multi}
K.~Hausman, Y.~Chebotar, S.~Schaal, G.~Sukhatme, and J.~Lim, ``Multi-modal
  imitation learning from unstructured demonstrations using generative
  adversarial nets,'' \emph{arXiv preprint}, vol. abs/1705.10479, 2017.

\bibitem{howard2009novel}
M.~Howard, S.~Klanke, M.~Gienger, C.~Goerick, and S.~Vijayakumar, ``A novel
  method for learning policies from variable constraint data,''
  \emph{Autonomous Robots}, vol.~27, no.~2, pp. 105--121, 2009.

\bibitem{laskeyrobot}
M.~Laskey, J.~Lee, C.~Chuck, D.~Gealy, W.~Hsieh, F.~T. Pokorny, A.~D. Dragan,
  and K.~Goldberg, ``Robot grasping in clutter: Using a hierarchy of
  supervisors for learning from demonstrations,'' \emph{Automation Science and
  Engineering (CASE), 2016 IEEE}, pp. 827--834, 2016.

\bibitem{laskey2017dart}
M.~Laskey, J.~Lee, R.~Fox, A.~Dragan, and K.~Goldberg, ``Dart: Noise injection
  for robust imitation learning,'' in \emph{Conference on Robot Learning},
  2017.

\bibitem{laskey2016shiv}
M.~Laskey, S.~Staszak, W.~Y.-S. Hsieh, J.~Mahler, F.~T. Pokorny, A.~D. Dragan,
  and K.~Goldberg, ``Shiv: Reducing supervisor burden in dagger using support
  vectors for efficient learning from demonstrations in high dimensional state
  spaces,'' in \emph{Robotics and Automation (ICRA), 2016 IEEE International
  Conference on}.\hskip 1em plus 0.5em minus 0.4em\relax IEEE, 2016, pp.
  462--469.

\bibitem{liang2017dvrk}
J.~Liang, J.~Mahler, M.~Laskey, P.~Li, and K.~Goldberg, ``Using dvrk
  teleoperation to facilitate deep learning of automation tasks for an
  industrial robot,'' in \emph{IEEE International Conference on Automation
  Science and Engineering (CASE)}, 2017.

\bibitem{lu2015human}
L.~Lu and J.~T. Wen, ``Human-directed robot motion/force control for contact
  tasks in unstructured environments,'' in \emph{International Conference on
  Automation Science and Engineering (CASE)}, 2015.

\bibitem{patel2017sprk}
V.~Patel, S.~Krishnan, A.~Goncalves, and K.~Goldberg, ``Sprk: A low-cost
  stewart platform for motion study in surgical robotics,'' in
  \emph{International Symposium on Medical Robotics (ISMR)}, 2018.

\bibitem{perez2017c}
C.~P{\'e}rez-D'Arpino and J.~A. Shah, ``C-learn: Learning geometric constraints
  from demonstrations for multi-step manipulation in shared autonomy,'' in
  \emph{IEEE International Conference on Robotics and Automation (ICRA)}, 2017.

\bibitem{rios2013derivative}
L.~M. Rios and N.~V. Sahinidis, ``Derivative-free optimization: a review of
  algorithms and comparison of software implementations,'' \emph{Journal of
  Global Optimization}, vol.~56, no.~3, pp. 1247--1293, 2013.

\bibitem{ross2010efficient}
S.~Ross and D.~Bagnell, ``Efficient reductions for imitation learning,'' in
  \emph{International Conference on Artificial Intelligence and Statistics},
  2010, pp. 661--668.

\bibitem{rossano2013easy}
G.~F. Rossano, C.~Martinez, M.~Hedelind, S.~Murphy, and T.~A. Fuhlbrigge,
  ``Easy robot programming concepts: An industrial perspective,'' in
  \emph{International Conference on Automation Science and Engineering (CASE)},
  2013.

\bibitem{scholkopf2001estimating}
B.~Sch{\"o}lkopf, J.~C. Platt, J.~Shawe-Taylor, A.~J. Smola, and R.~C.
  Williamson, ``Estimating the support of a high-dimensional distribution,''
  \emph{Neural computation}, vol.~13, no.~7, pp. 1443--1471, 2001.

\bibitem{scholkopf2002learning}
B.~Sch{\"o}lkopf and A.~J. Smola, \emph{Learning with kernels: Support vector
  machines, regularization, optimization, and beyond}.\hskip 1em plus 0.5em
  minus 0.4em\relax MIT press, 2002.

\bibitem{schroecker2017state}
Y.~Schroecker and C.~L. Isbell, ``State aware imitation learning,'' in
  \emph{Advances in Neural Information Processing Systems}, 2017, pp.
  2915--2924.

\bibitem{schulman2015trust}
J.~Schulman, S.~Levine, P.~Abbeel, M.~Jordan, and P.~Moritz, ``Trust region
  policy optimization,'' in \emph{International Conference on Machine Learning
  (ICML)}, 2015.

\bibitem{singh2017gplac}
A.~Singh, L.~Yang, and S.~Levine, ``Gplac: Generalizing vision-based robotic
  skills using weakly labeled images,'' \emph{arXiv preprint}, vol.
  abs/1708.02313, 2017.

\bibitem{thananjeyan2017multilateral}
B.~Thananjeyan, A.~Garg, S.~Krishnan, C.~Chen, L.~Miller, and K.~Goldberg,
  ``Multilateral surgical pattern cutting in 2d orthotropic gauze with deep
  reinforcement learning policies for tensioning,'' in \emph{IEEE International
  Conference on Robotics and Automation (ICRA)}, 2017.

\bibitem{todorov2012mujoco}
E.~Todorov, T.~Erez, and Y.~Tassa, ``Mujoco: A physics engine for model-based
  control,'' in \emph{International Conference on Intelligent Robots and
  Systems (IROS)}, 2012.

\bibitem{van2010superhuman}
J.~Van Den~Berg, S.~Miller, D.~Duckworth, H.~Hu, A.~Wan, X.-Y. Fu, K.~Goldberg,
  and P.~Abbeel, ``Superhuman performance of surgical tasks by robots using
  iterative learning from human-guided demonstrations,'' in \emph{ICRA, 2010
  IEEE}.\hskip 1em plus 0.5em minus 0.4em\relax IEEE, 2010, pp. 2074--2081.

\bibitem{vert2006consistency}
R.~Vert and J.-P. Vert, ``Consistency and convergence rates of one-class svms
  and related algorithms,'' \emph{The Journal of Machine Learning Research},
  vol.~7, pp. 817--854, 2006.

\bibitem{yip2016model}
M.~C. Yip and D.~B. Camarillo, ``Model-less hybrid position/force control: a
  minimalist approach for continuum manipulators in unknown, constrained
  environments,'' \emph{IEEE Robotics and Automation Letters}, vol.~1, no.~2,
  pp. 844--851, 2016.

\end{thebibliography}

\end{document}